\documentclass[11pt]{article}
\usepackage[T1]{fontenc}
\usepackage[utf8]{inputenc}
\usepackage{lmodern}
\usepackage{amsmath,amssymb,amsthm,booktabs}
\usepackage[margin=1in]{geometry}
\usepackage{array,tabularx}
\usepackage{flafter}
\usepackage{fancyvrb}
\usepackage{tikz}
\usetikzlibrary{arrows.meta}
\usepackage{microtype}
\usepackage[colorlinks=true,linkcolor=blue,citecolor=blue,urlcolor=blue]{hyperref}

\makeatletter
\renewcommand{\@cite}[2]{[\if@tempswa #2, \fi #1]}
\makeatother
\newtheorem{theorem}{Theorem}[section]

\newtheorem{proposition}[theorem]{Proposition}

\theoremstyle{definition}
\newtheorem{definition}[theorem]{Definition}
\theoremstyle{remark}

\numberwithin{equation}{section}
\newcommand{\psep}{\par\vskip 1mm}

\newcommand{\TR}{\operatorname{TR}}
\newcommand{\ind}{\mathbf 1}
\newcommand{\Pred}{\operatorname{Pred}}
\newcolumntype{Y}{>{\raggedright\arraybackslash}X}
\hypersetup{pdftitle={Dependency Graphs of Mathematical Arguments: Representation, Structural Measurement, and Proof Review},pdfauthor={Eunji Park},pdfsubject={Mathematical dependencies, argument structure, and mathematical knowledge management}}
\title{Dependency Graphs of Mathematical Arguments:\\
Representation, Structural Measurement, and Proof Review}
\author{Eunji Park\\[0.4em]
{\small Graduate School of Data Science, Seoul National University}\\
{\small Gwanak-ro 1, Gwanak-gu, Seoul 08826, Korea}\\[0.3em]
{\small Email address: \texttt{dmswl4743@snu.ac.kr}}}
\date{}
\begin{document}
\maketitle

\begin{abstract}
We formulate a source-grounded framework for representing and measuring
mathematical arguments with assistance from large language models.
The inventory includes named statements and numbered equations, as well
as unnumbered equations, local assumptions, constructions, and prose
claims needed to represent a selected argument. Source occurrences are
distinguished from contextual states and inference records; proposed
dependencies retain evidence and review status. A measurement contract
fixes manuscript versions, boundaries, granularity, proof choices, and
macro anchors. On a selected finite directed acyclic graph, depth is
the longest prerequisite-path length, and bulk counts distinct ancestor
edges. We give a union recurrence preserving shared subarguments and
establish that fixed-vertex reachability determines depth, whereas
fixed-anchor reachability determines the macro skeleton. Unrestricted
refinement need not preserve depth or bulk.
The principal corpus-level objective is to analyze recurring groups of
jointly used theorems, infrequent combinations, and results connecting
different dependency communities. These patterns motivate probabilistic
premise suggestions and targeted review of possible omissions or
unchecked application conditions in human-written and AI-generated
proofs. Suggested premises require mathematical validation; statistical
association does not certify a proof. Further applications include
source-linked reading support, cross-paper theorem exploration, and
comparison of individual and collaborative argument shapes, with problem
structure and writing roles treated as competing explanations.
Measurements relative to specified axioms additionally require resolved
foundational dependencies and fixed inference conventions. A versioned
prompt protocol and an evaluation design separate annotation discovery,
deterministic recomputation, and review assistance. These applications
are proposed research directions rather than reported empirical findings.
The measures describe a specified representation, not intrinsic
mathematical difficulty, originality, or importance.
\end{abstract}


\section{Introduction}\label{sec:intro}

\subsection{Historical Background}
The organization of mathematical knowledge has both a logical and an
expository dimension. A result depends on assumptions, definitions, and
earlier arguments, but the order in which a paper presents these ingredients
need not coincide with their dependency order. A calculation in an appendix
may support a theorem in the main text, several results may reuse one
construction, and a paper may contain independent branches. A sequence of
pages therefore provides only one view of the argument.
\psep
Making mathematical structure explicit has a long history. De Bruijn's
Automath language represented mathematical objects and assertions through
explicitly introduced names and categories, with texts designed for
computer checking \cite[Section 1]{deBruijn}. At the level of mathematical
writing, Lamport advocated hierarchical proofs in which named steps expose
the organization of an argument and allow readers to move between an
outline and its details \cite[pp.~601 and 607]{Lamport}. These approaches
address different questions, but both show why the structure supporting a
statement deserves an explicit representation.
\psep
Mathematical knowledge management extends this concern from individual
proofs to reusable bodies of mathematics. Theory graphs, as discussed by
Kohlhase in the setting of OMDoc/MMT, represent mathematical theories and
the interpretations connecting them \cite[Section 3]{Kohlhase}.
Bibliometric networks offer a different perspective: Small's co-citation
analysis records how frequently two documents are cited together
\cite{Small}. Such document-level relations describe the organization of
literature, but they do not specify which theorem, assumption, or
intermediate calculation is used in a particular proof. The present study
is concerned with that finer level, while retaining links to the source
documents and their broader mathematical context.

\subsection{Related Work}
Alama, Mamane, and Urban extract detailed dependencies from Coq and Mizar
and distinguish dependencies relevant to logical justification from those
needed for the successful processing of formal texts
\cite[Introduction and Section 2]{Alama}. This is a direct methodological precedent for
specifying what an edge means. Their work also cautions against treating
all dependencies produced by a software system as interchangeable. Our
setting is primarily informal mathematical writing, where identifying a
dependency additionally requires interpreting the source and recording
the evidence for that interpretation.
\psep
NaturalProofs provides a corpus of natural-language mathematical statements
and proofs, an explicit statement-level reference graph, and
reference-retrieval and reference-generation tasks
\cite[Sections 3--4 and Table 2]{NaturalProofs}. Its reference graph is an important
precedent for representing relations within informal mathematics; the
authors note that it can contain cycles. Its evaluation tasks concern
retrieving and generating references used in a proof, rather than a
complete decomposition of internal calculations or the stability of
depth and bulk under such a decomposition. Reference links can supply
candidate inputs to our framework, but contextual applicability and a
compatible acyclic proof selection still have to be established.
\psep
Lean Blueprint supports author-declared dependencies separately for
statements and proofs and links the exposition to named Lean declarations
\cite[``Editing the blueprint'']{Blueprint}. KnowTeX develops a closely
related approach directly from \LaTeX{} sources, with manual annotations,
rule-based dependency inference, interactive edge review, cycle detection,
and transitive reduction \cite[Sections 2 and 4.3--4.6]{KnowTeX}. These are close
precedents for a Paper Graph, including its use as a reading map. We
therefore do not claim that source-linked dependency visualization,
human review of proposed edges, or educational use of such graphs is new.
KnowTeX also recognizes constructions, remarks, and custom environments,
including items without explicit labels, but excludes equation and align
environments from its node parser \cite[Section 4.2]{KnowTeX}.
Its order-based inference rules propose dependencies requiring separate
review \cite[Section 4.4]{KnowTeX}. Our emphasis is on a measurement
contract that fixes source evidence, decomposition, and the treatment of
shared subarguments and unresolved components.
\psep
TheoremGraph combines extractor-labeled candidate dependencies with
semantic links between related statements within and across informal
and formal mathematical corpora
\cite[Sections 3--6 and 10]{TheoremGraph}. It is a direct
precedent for the cross-paper scope of a Math Graph. The distinction
between a candidate match, a semantic similarity, and an established
application of a theorem is essential here: an application requires the
appropriate statement, context, substitution, and hypotheses. Our proposal
organizes these distinctions for argument measurement; it does not claim
priority for a statement-level mathematical network or for connecting
informal and formal corpora.
\psep
Li, Peng, Severini, and Shafto analyze the network structure of Mathlib,
including its dependency layers, depth, and community structure
\cite[Sections 4--5 and Appendices B--C]{Mathlib}. Their analysis emphasizes
that logical, mathematical, and infrastructural relations can produce
different network interpretations. Its declaration graph points from a
declaration to its premises, whereas our measurement graph points from
inputs to dependent outputs \cite[Appendix B.2.1]{Mathlib}. Node-level
depth and degree comparisons therefore require explicit alignment of
edge conventions. The work is particularly relevant to our proposed
study of structural shape and connections between fields.
The distinction between statement and proof dependencies, and the
warning against removing actual proof-use edges by transitive reduction,
are particularly close to our separation of displayed structure and
edge-count measurements \cite[Appendices B.2.2 and C.2.3]{Mathlib}.
Depth and network centrality are thus established analytical tools, not
new invariants introduced here. Likewise, our macro-level construction
uses the classical transitive reduction of a directed graph
\cite{Aho}.
\psep
Finer representations also have direct precedents. Proofscape places
assertions, assumptions, and definitions from informal proofs on nodes
and supports interactive expansion of inference steps and cited results
\cite[Section 1, pp.~57--58]{KiefferProofscape}.
Mathematical Derivation Graphs studies equation dependencies with nodes
explicitly restricted to numbered equations and evaluates LLM extraction
\cite[Sections 2 and 5.2]{MDG}.
ProofFlow uses an LLM to extract conditions, definitions, intermediate
claims, and conclusions as a proof DAG; its examples include unnumbered
calculations and constructions, and it discusses step granularity
\cite[Sections 3.1 and 4.1, Figures 1--2]{ProofFlow}.
Thus, neither broad proof-step nodes nor LLM-assisted decomposition is
claimed as a contribution in isolation. Our coverage requirement is to
retain argument-bearing unnumbered material throughout the selected
manuscript scope, under an explicit and auditable measurement policy.
\psep
Learning useful premises from proof corpora is likewise established.
Alama et al. combine proof-dependency analysis with learned premise
selection and automated re-verification \cite{AlamaPremise}.
NaturalProofs studies conditional reference generation
\cite[Section 5]{NaturalProofs}, while TheoremGraph uses dependency-neighbor
expansion to supplement retrieval \cite[Appendix E.2]{TheoremGraph}.
These are direct antecedents of corpus-based proof assistance.
Our proposed application retains statistical suggestions as a separate
relation type: only source-supported, contextually reviewed uses enter
the selected measurement graph.
\psep
Taken together, these works supply substantial parts of the surrounding
infrastructure. The contribution pursued here is a unified, source-grounded
protocol for measuring a chosen argument, together with precise conditions
for comparison and a research design for studying the resulting data.
This integration specifies the method being proposed; it does not
establish priority for its individual components or superior extraction
accuracy. Any advantage over existing tools must be established on common
annotated material.

\subsection{Main Question}
Our main question is the following:
\begin{quote}
How can the dependencies actually used in a mathematical paper be
represented and measured reproducibly across LLM-assisted analyses, so
that both its internal argument structure and its connections to other
papers can be compared under explicitly controlled conditions?
\end{quote}
This question has three parts. First, which source occurrences and
intermediate states should be distinguished, and what evidence makes a
proposed dependency admissible? Second, which structural properties are
preserved when a representation is simplified or refined? Third, once
comparable graphs have been collected, which patterns in learning,
exposition, collaboration, and theorem reuse can be investigated without
confusing structural measurements with mathematical value?
\psep
The object being measured is a recorded argument under a specified
representation. It is not a theorem independently of all its proofs. A
single statement may admit several proofs, and the same proof may be
expanded to different levels of detail. Even a fixed graph can support
different questions: a longest prerequisite chain describes sequential
dependence, whereas the union of all prerequisites describes the amount
of recorded structure supporting a target. Neither quantity alone
determines how difficult the argument is to understand or discover.

\subsection{Proposed Framework and Methods}
We use the term \emph{Paper Graph} for a source-linked representation of
definitions, assumptions, constructions, intermediate calculations, and
results in one manuscript. Its data have three levels: document
occurrences preserve what appears in the source; contextual states and
inference records encode a chosen argument; and a macro skeleton records
reachability between selected anchor states. The term \emph{Math Graph}
denotes a collection of such representations connected by identified
theorem applications, with versions and contexts retained.
\psep
Large language models may assist in identifying unnumbered equations,
local assumptions, object constructions, compatibility checks, and claims
expressed in prose, in addition to explicitly labeled items. The inclusion
criterion is an item's role in the recorded justification, rather than
its typography or position. Each proposed item and dependency retains
its source location, context, and evidential status for independent review.
A complete inventory of numbered statements alone is therefore not a
complete representation of the argument.
\psep
The source-linked reading view can keep these items accessible while
presenting a compact overview. A separate computational representation
records the substantive steps needed for measurement; screen layers and
displayed item counts do not determine the structural metrics.
\psep
The central methodological device is a measurement contract
$\mathcal C=(M,B,P,W,H)$. It fixes the manuscript and scope $M$, the accepted
boundary inputs $B$, the granularity policy $P$, a consistent proof
selection $W$, and the macro anchors $H$. Under this contract, a target
$v$ in a finite directed acyclic graph has depth $D(v)$, measured in edges
along a longest incoming path, and bulk $S(v)$, measured by distinct edges
in its ancestor subgraph. Shared subarguments are counted by set union,
not by adding the sizes of branches. We also consider the edge-to-node
ratio $q(G)=|E|/|V|$ as a first summary of dependency per node, while keeping
its interpretation separate from general graph complexity.
\psep
A central design objective is to make the extracted dependency structures
and their measurements as stable as possible across different LLMs.
With the source and measurement contract held fixed, the choice of
extractor should have as little influence as possible on which argument
steps and actual uses are recorded and on the resulting $D$ and $S$.
This objective includes repeated runs of one model and comparisons
between models; it is an empirical requirement to evaluate, not a
guarantee supplied by identical prompts.
\psep
Such comparisons require source- and context-preserving alignment of
states under the same granularity and proof-selection policies.
Agreement on states, dependency edges, anchor relations, and target-level
$D$ and $S$ should be reported separately: equal measurements do not
imply identical graphs, and model agreement does not establish
mathematical validity. Disagreements should remain traceable to the
source for review. Once the same accepted graph is fixed, deterministic
recalculation must give the same measurements independently of the LLM
that proposed its annotations. Sections~\ref{sec:algorithms} and
\ref{sec:evaluation} distinguish this computational requirement from
stability in discovering the graph itself.
\psep
The structural results establish the scope of these definitions.
Proposition~\ref{prop:union} gives the union recurrence;
Propositions~\ref{prop:depth}--\ref{prop:bounds} address reachability,
refinement, and conditional bounds; and Proposition~\ref{prop:macro}
characterizes the macro skeleton. These are elementary consequences of
finite DAG structure. Their role is to justify the measurement protocol,
not to claim new general graph theory. The remaining methods concern
source annotation, explicit uncertainty, independent recalculation, and
evaluation against expert reference data. Section~\ref{sec:prompt} introduces
PGPP~0.1, a proposed standardized prompt protocol that makes these
requirements operational. Its complete reusable prompt and output contract
are given in Appendix~\ref{app:prompt}.

\subsection{Future Directions}
Our principal corpus-level objective is to study the joint use of
mathematical results as evidence for directing proof review. A collection
of reviewed Paper Graphs could reveal theorem groups repeatedly used
together, combinations rarely observed under comparable opportunities
for use, and results connecting different dependency communities.
These patterns concern actual uses in selected arguments, rather than
bibliography-level co-citation alone. Their interpretation must account
for theorem frequency, proof size, subject area, and corpus coverage.
\psep
The central proof-review question is therefore:
\begin{quote}
Given a human-written or AI-generated proof, can reviewed, source-grounded
dependency records and theorem co-use patterns jointly improve the
prioritization of premises, intermediate claims, and theorem applications
that require mathematical checking?
\end{quote}
The proposed input consists of the target statement and proof, its local
dependency annotations with unresolved steps retained, and a separate
corpus of reviewed proof units. Comparable prior arguments could suggest
premise bundles to inspect; bridge nodes could extend the search across
subject communities. The output would be a ranked set of specific checks:
an application whose hypotheses have not been established, an intermediate
claim whose recorded inputs leave its justification unresolved, or a
potentially omitted premise worth investigating. Each recommendation
would identify the local passage, the relevant theorem instance and
conditions, and supporting examples from the corpus.
\psep
The empirical goal is to improve the detection and resolution of actual
proof defects under a fixed review budget. Comparisons should include
unaided review, statement-similarity retrieval, and suggestions based
only on theorem frequency, with the same source access and candidate
availability. Outcomes include independently adjudicated defects found
and resolved, false alarms, and reviewer effort. A common pattern does
not certify a proof, and an unusual pattern need not be incorrect.
Suggested premises must be checked against the statement, substitution,
and hypotheses; a proposed repair remains separate from the original
source graph until a revised proof is justified. Section~\ref{sec:applications}
specifies the prediction task, and Section~\ref{sec:evaluation} describes
how to test whether these suggestions improve mathematical review.
\psep
A first practical application is as a companion to reading a paper.
A reader could select a result, inspect all its prerequisites, distinguish
familiar inputs from unfamiliar ones, and return from a graph item to
the source. Whether this improves understanding requires evaluation;
logical dependency order need not be the best teaching order.
\psep
Verified theorem applications would connect individual Paper Graphs
into a larger Math Graph. Expanding external results could eventually
support depth relative to specified axioms and inference rules, provided
the necessary foundations, definitions, versions, and proof choices have
been connected explicitly. Otherwise depth remains relative to the
declared boundary of the current analysis.
\psep
The data also permit hypotheses about recurring argument shapes in
single-author and collaborative work. A collaboration may combine
distinct approaches or reflect the person organizing the exposition.
The mathematical problem itself may impose much of the shape. Comparisons
must therefore separate, where possible, author-associated organization,
writing roles, subject matter, and problem structure. They concern
recorded arguments, not personal character or mathematical ability.
\psep
Sections~\ref{sec:mathgraph}--\ref{sec:applications} specify these proposed
applications, and Section~\ref{sec:evaluation} gives their evaluation
design. No corpus-level patterns or benefits are reported as findings.
\psep
To help readers try the framework on a small selection of papers and
obtain a visual overview, we provide two complementary prompts for
practical convenience. The PGPP~0.1 reference prompt, introduced in
Section~\ref{sec:prompt} and reproduced in Appendix~\ref{app:prompt},
guides the construction of source-linked argument records. The
supplementary website prompt in Section~\ref{sec:viewer-prompt} uses
the resulting graph data to build an interactive reading view, with
complete prerequisite highlighting and persistent click selection.
Readers can first apply the analysis prompt to their selected manuscripts
and then supply the graph records to the interface prompt to explore
the results visually. This provides a practical route from a few source
papers to an inspectable map; rapid visualization does not itself
establish completeness or correctness, and annotation status and
unresolved items should remain visible.

\subsection*{Disclosure of AI assistance}
The original research ideas and conceptual direction of this paper
originated with the author. GPT-6 Astra was used in developing the
mathematical details and proofs that make these ideas precise.
GPT-5.6 Sol was used extensively in drafting and revising the manuscript,
improving the exposition, and correcting grammar. Responsibility for the
final content, including the mathematical statements, proofs, and accuracy
of the references, rests with the author.

\section{Representation and Data Model}\label{sec:model}

\subsection{Source occurrences and contextual states}
A source occurrence is a particular item in a particular manuscript
version. The inventory includes labeled definitions, assumptions, lemmas,
propositions, theorems, corollaries, and numbered equations. Unnumbered
equations, local assumptions, constructions, calculations, and prose
claims are included when they play a role in the selected argument. Each occurrence retains its text or expression,
item type, label, and source location. Document order is metadata rather
than a dependency rule.
\psep
A computational state is a claim or expression in a context. A state $v$
records a context $\Gamma_v$, a claim or expression $\varphi_v$, object
identifiers, parameter substitutions, and relevant conventions. Identical
printed formulas in different contexts need not represent the same state.
Conversely, exact reprints of the same calculation may be distinct source
occurrences pointing to one state.
\psep
The correspondence between occurrences and states is a relation, not
necessarily a bijection. One displayed equation may contain several
subcalculations, and one state may be displayed more than once. A component
of a displayed tuple or family is not identical to the entire family.
Its correspondence must be recorded separately, or a component-extraction
step must be introduced. This distinction preserves the source while
making any sharing decision auditable.

\subsection{Actual use and inference records}
An edge $u\to v$ means that $u$ is used as an input in the selected
justification, definition, or construction of $v$. It does not assert that
every possible proof of $v$ must use $u$. Citations, semantic similarities,
equivalences, and layout connections have separate relation types.
\psep
Pairwise edges alone do not record the difference between joint premises
and alternative proofs. We therefore retain inference records
\begin{equation}\label{eq:inference}
 r=(I_r,o_r,\tau_r,\theta_r,\varepsilon_r,w_r).
\end{equation}
Here $I_r$ is a finite set of joint input states, $o_r$ is the output,
$\tau_r$ is an operation type, $\theta_r$ specifies context and substitution,
$\varepsilon_r$ contains source evidence and review information, and $w_r$
identifies a proof alternative. Inputs within one record have an AND
interpretation. Alternative derivations of the same output have an OR
interpretation and are not combined into one set of required premises.
\psep
For measurement, a proof selection $W$ chooses one compatible generating
record for each derived state in the analyzed argument; accepted boundary
states have no generating record. Compatibility requires agreement on
shared states, contexts, and any local choices used by the selected
records. Incompatible derivations must be kept as separate proof instances.
The selected records induce a simple directed graph by placing $u\to o_r$
for each $u\in I_r$. The records themselves are retained, since this
projection does not recover their full semantics.
\psep
Repeated appearances of one input in a single record can be stored as
additional metadata, but the basic graph contains each ordered pair only
once. The depth of this state graph is not the depth of an incidence graph
with extra operation vertices. Switching between those representations
requires an explicit change of measurement policy.

\subsection{Boundaries and acyclicity}
The boundary $B$ specifies the assumptions, primitives, external results,
or unresolved inputs accepted as starting points. Giving an external
theorem depth zero means that its proof is outside the current expansion,
not that the theorem is elementary. A block whose inputs and output are
known but whose interior has not been decomposed remains an explicit
opaque block.
\psep
We restrict the measurement graph to a finite DAG. If a cycle appears,
one must first check whether assumptions and conclusions were conflated,
equivalence links were treated as dependencies, or different proof
alternatives or manuscript versions were mixed. Induction is represented
through a rule with its local assumptions and their discharge, rather
than by a self-dependent theorem node. Recursive definitions similarly
require their justification conditions. An unresolved cycle is reported
as a representation problem; edges are not silently removed to obtain a
numerical answer.

\section{A Source-Grounded Measurement Protocol}\label{sec:protocol}

\subsection{Fixing the measurement contract}
The contract $\mathcal C=(M,B,P,W,H)$ fixes the conditions under which
graphs and measurements can be compared. Depth and bulk depend on the
first four components; the macro skeleton additionally depends on $H$.
The manuscript, annotation set, operation registry, and calculation code
should have version identifiers and content hashes.
\psep
The granularity policy $P$ specifies what counts as a substantive step.
It does not claim to identify a universal smallest unit of informal
mathematics. Its purpose is to make a choice repeatable within a stated
domain and analytical objective. A domain-specific extension receives a
new version. Measurements obtained under different registries should not
be treated as directly comparable without a correspondence analysis.

\subsection{Substantive decomposition}
A step boundary is justified when an output is used by another inference,
requires a separate justification, or is identified by a registered
operation rule. Independent calculations remain parallel. A sequential
edge is inserted only when one step consumes the output of another.
Line breaks and equation numbers alone do not define computational steps.
Table~\ref{tab:operations} gives a starting registry.

\begin{table}[htbp]
\centering\small
\caption{An illustrative operation registry. Each entry must be specialized
by its allowed inputs, outputs, and context.}\label{tab:operations}
\renewcommand{\arraystretch}{1.15}
\begin{tabularx}{\textwidth}{@{}p{0.17\textwidth}YY@{}}
\toprule
Operation & One recorded step & Distinction to preserve\\
\midrule
SUB & A specified simultaneous substitution & Do not count one step per substituted symbol.\\
APPLY & An application of a theorem or identity & Fix the instance and its hypotheses; mere mention is insufficient.\\
EXPAND & Expansion under a registered rule & Printed length is not the number of steps.\\
COLLECT & Collection or cancellation in a fixed normal form & Do not count the same normalization twice.\\
COMPONENT & Extraction of a used component or coefficient & A component is not the entire expression family.\\
CONSTRUCT & Construction of an object or map & Retain its defining conditions and source evidence.\\
COMPAT & A separately justified compatibility condition & A declaration is not an established condition.\\
COMBINE & An inference combining established outputs & Renaming or reprinting is not a new combination.\\
OPAQUE & A block with recorded inputs and output & The unresolved interior is not reconstructed by assumption.\\
\bottomrule
\end{tabularx}
\end{table}

Each decomposed step records its location, inputs, output, reason for
separation, and evidence status. Explicit source calculations are
distinguished from interpretive decompositions grounded in the source.
A change of notation can be substantive when it invokes an identity or
condition. A long expression can remain one step when the registry treats
its simultaneous transformation as one operation.

\subsection{Extraction, review, and shared calculations}
The workflow proceeds from a source inventory to contextual states,
candidate dependencies, internal decomposition, independent review, and
an accepted measurement graph. An LLM can assist with the inventory,
local interpretation, and decomposition, but its output is a proposal
with explicit source evidence, not independent acceptance. The inventory of numbered items can be
checked against the manuscript. Its completeness does not establish the
completeness of unnumbered claims or implicit dependencies.
\psep
Evidence status and review status are separate. Evidence may record
explicit use, a locally justified interpretation, or an unresolved
candidate. Review status may be proposed, accepted, rejected, or unresolved.
The presence of a source location does not establish that the cited
passage supports the edge. Agreement among extraction systems likewise
does not replace a semantic assessment.
\psep
Sharing is allowed only when the context, objects, substitutions, inputs,
and calculation instance agree. Exact reprints may share a state while
remaining distinct document occurrences. Equivalent propositions proved
in different ways are not merged merely because they have the same truth
conditions. If a proposed identification creates self-dependence or a
cycle, the identification must be reconsidered.

\section{Depth, Bulk, and Dependency per Node}\label{sec:metrics}

\subsection{Ancestor subgraphs and basic quantities}
Let $G=(V,E)$ be the selected finite DAG, with edges directed from inputs
to outputs. Write $\Pred_G(v)$ for the direct predecessors of $v$, and let
$u\leadsto_G v$ denote a nonempty directed path. Define
\begin{align}
 A_G(v)&=\{v\}\cup\{u\in V:u\leadsto_G v\},\label{eq:ancestors}\\
 F_G(v)&=E\cap\bigl(A_G(v)\times A_G(v)\bigr).\label{eq:ancestor-edges}
\end{align}
Thus $F_G(v)$ is the edge set of the subgraph induced by all ancestors of
$v$, including $v$ itself.

\begin{definition}[Depth and bulk]\label{def:metrics}
The depth of $v$ is the length, in edges, of a longest directed path ending
at $v$. Its bulk is the number of distinct edges in its ancestor subgraph:
\begin{align}
 D_G(v)&=\begin{cases}
 0,&\Pred_G(v)=\varnothing,\\
 1+\max\limits_{u\in\Pred_G(v)}D_G(u),&\Pred_G(v)\ne\varnothing,
 \end{cases}\label{eq:depth}\\
 S_G(v)&=|F_G(v)|.\label{eq:bulk}
\end{align}
\end{definition}

When comparison conditions matter, these quantities can be written
$D_{M,B,P,W}(v)$ and $S_{M,B,P,W}(v)$. We suppress the contract only when
it is fixed. Depth describes sequential dependence; bulk describes the
amount of distinct recorded dependency structure supporting a target.
Bulk is neither the number of inference records nor the number of
arithmetic operations executed by a machine.
\psep
A four-edge chain, a branching argument with a shared input, and a
conclusion with four independent inputs all have bulk four but different
depths; see Figure~\ref{fig:examples}. These are abstract illustrations
of the definitions, not measurements extracted from particular papers.

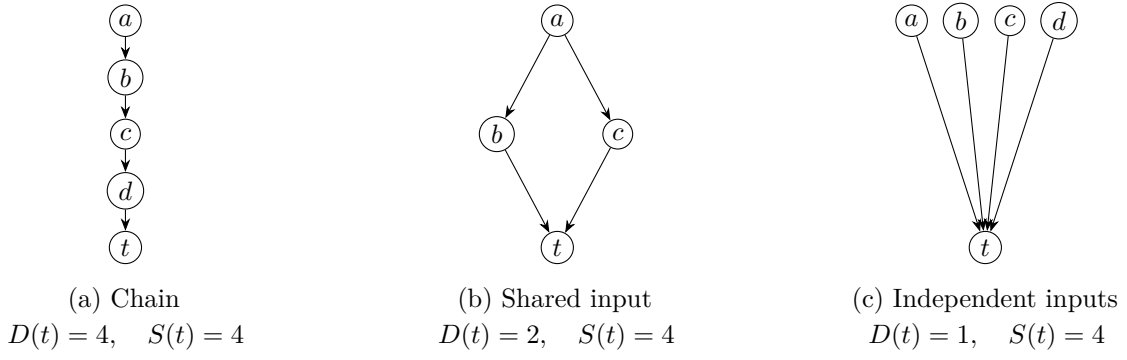
\begin{figure}[htbp]
\centering
\begin{minipage}[t]{0.31\textwidth}\centering
\begin{tikzpicture}[x=1cm,y=1cm,>=Stealth,
 every node/.style={circle,draw,fill=white,inner sep=1.8pt,font=\small}]
\node (a) at (0,3) {$a$}; \node (b) at (0,2.25) {$b$};
\node (c) at (0,1.5) {$c$}; \node (d) at (0,.75) {$d$};
\node (t) at (0,0) {$t$};
\draw[->] (a)--(b); \draw[->] (b)--(c); \draw[->] (c)--(d); \draw[->] (d)--(t);
\end{tikzpicture}\\[4pt]
{\small (a) Chain\\$D(t)=4,\quad S(t)=4$}
\end{minipage}\hfill
\begin{minipage}[t]{0.31\textwidth}\centering
\begin{tikzpicture}[x=1cm,y=1cm,>=Stealth,
 every node/.style={circle,draw,fill=white,inner sep=1.8pt,font=\small}]
\node (a) at (.8,3) {$a$}; \node (b) at (0,1.5) {$b$};
\node (c) at (1.6,1.5) {$c$}; \node (t) at (.8,0) {$t$};
\draw[->] (a)--(b); \draw[->] (a)--(c); \draw[->] (b)--(t); \draw[->] (c)--(t);
\end{tikzpicture}\\[4pt]
{\small (b) Shared input\\$D(t)=2,\quad S(t)=4$}
\end{minipage}\hfill
\begin{minipage}[t]{0.31\textwidth}\centering
\begin{tikzpicture}[x=1cm,y=1cm,>=Stealth,
 every node/.style={circle,draw,fill=white,inner sep=1.8pt,font=\small}]
\node (a) at (0,3) {$a$}; \node (b) at (.65,3) {$b$};
\node (c) at (1.3,3) {$c$}; \node (d) at (1.95,3) {$d$};
\node (t) at (.975,0) {$t$};
\draw[->] (a)--(t); \draw[->] (b)--(t); \draw[->] (c)--(t); \draw[->] (d)--(t);
\end{tikzpicture}\\[4pt]
{\small (c) Independent inputs\\$D(t)=1,\quad S(t)=4$}
\end{minipage}
\caption{Equal bulk does not determine depth. The arrows record selected
input use, and the displayed measurements concern the final state $t$.}
\label{fig:examples}
\end{figure}

\subsection{Union calculation and shared structure}
The ancestor edge sets satisfy
\begin{equation}\label{eq:union}
 F_G(v)=\bigcup_{u\in\Pred_G(v)}\bigl(F_G(u)\cup\{(u,v)\}\bigr).
\end{equation}
An empty union is empty. Computing unions, rather than adding predecessor
bulks, prevents a shared subargument from being counted repeatedly.

\begin{proposition}[Correctness of the union recurrence]\label{prop:union}
Processing a finite DAG in topological order using
\eqref{eq:depth} and \eqref{eq:union} computes the longest incoming path
length and the distinct ancestor edge set of every vertex.
\end{proposition}
\begin{proof}
At a source, there are no incoming edges and both statements hold. Every
nonempty path ending at $v$ has a final edge $(u,v)$ with
$u\in\Pred_G(v)$, giving the recurrence for depth. An edge in the ancestor
subgraph of $v$ either ends at $v$ or lies on a path to one of its direct
predecessors, so it belongs to the right-hand side of \eqref{eq:union}.
Conversely, every edge on that right-hand side lies in the ancestor
subgraph of $v$. Induction in topological order proves the claims.
Set union retains each shared edge only once.
\end{proof}

For several targets $T\subseteq V$, the distinct supporting edges are
$F_G(T)=\bigcup_{t\in T}F_G(t)$. In general,
$|F_G(T)|\ne\sum_{t\in T}S_G(t)$. The difference between the latter sum
and the union size counts repeated edge incidences across target ancestor
subgraphs. It is a description of overlap, not a new inference count.

\subsection{Reachability and changes of granularity}
\begin{proposition}[Depth is determined by fixed-vertex reachability]\label{prop:depth}
Two finite DAGs on the same vertex set with the same reachability relation
have the same depth at every vertex.
\end{proposition}
\begin{proof}
Reachability defines a strict partial order. Every directed path ending
at $v$ is a chain in that order. Conversely, take a longest chain ending
at $v$. Consecutive elements of the chain are related by a cover relation;
otherwise an intermediate element could extend the chain. A cover
relation must be represented by a direct edge in either DAG, since a
longer realizing path would supply an intermediate element. Thus a
longest chain is a path in each graph, and its length is determined by
reachability alone.
\end{proof}

Transitive reduction therefore preserves depth on a fixed vertex set,
but it need not preserve bulk. Adding $a\to c$ to $a\to b\to c$ leaves
$D(c)=2$ and changes $S(c)$ from two to three. If $a$ is actually used
directly in the selected justification of $c$, the additional edge belongs
in the basic measurement graph even if a simplified visualization omits it.

\begin{proposition}[Failure of invariance under arbitrary refinement]\label{prop:refinement}
Depth and bulk are not invariant under arbitrary subdivision of internal
argument steps, even when reachability between the original boundary
vertices is unchanged.
\end{proposition}
\begin{proof}
Replace the single edge $a\to b$ by
$a\to x_1\to\cdots\to x_k\to b$. Both the depth and bulk of $b$ change
from one to $k+1$, while $a$ still reaches $b$. The vertex set has changed,
so Proposition~\ref{prop:depth} does not apply.
\end{proof}

This elementary example explains why a granularity policy is necessary.
It does not say that every refinement changes a graph's maximum depth:
subdividing a noncritical branch can leave the longest path unchanged.
An increase in recorded steps and an increase in depth are distinct events.

\subsection{Conditional bounds for unresolved dependencies}
\begin{proposition}[Monotonicity on a fixed vertex set]\label{prop:bounds}
Suppose $G^-=(V,E^-)$, $G=(V,E)$, and $G^+=(V,E^+)$ satisfy
$E^-\subseteq E\subseteq E^+$, and $G^+$ is a DAG. Then, for every $v$,
\begin{equation}\label{eq:bounds}
\begin{split}
 D_{G^-}(v)&\le D_G(v)\le D_{G^+}(v),\\
 S_{G^-}(v)&\le S_G(v)\le S_{G^+}(v).
\end{split}
\end{equation}
\end{proposition}
\begin{proof}
Adding edges removes neither existing paths nor existing ancestors and
edges in an ancestor subgraph. Hence neither quantity decreases.
\end{proof}

The proposition provides bounds only if the inclusion assumptions hold
for the intended graph. It does not turn the minimum and maximum values
returned by several extraction systems into a confidence interval. An
upper graph containing mutually exclusive alternatives may give a
numerical upper bound, if acyclic, without representing any realizable
proof. Changes in state correspondence or granularity require a separate
analysis and are not covered by \eqref{eq:bounds}.

\subsection{Dependency per node and structural interconnection}
For a nonempty graph, define
\begin{equation}\label{eq:q}
 q(G)=\frac{|E|}{|V|}.
\end{equation}
This is both the average indegree and the average outdegree. It is a
natural operational interpretation of ``dependency per node'' and a
first summary of how many direct relations the representation records
per state. We call it the \emph{dependency-per-node ratio}. If the informal
phrase ``connectivity complexity'' is used for this quantity, it should
refer only to this specified ratio, not to a general measure of
mathematical complexity or connectedness.
\psep
For example, a chain and a star on the same number $n$ of vertices both
have $q=(n-1)/n$. The ratio does not distinguish their sequential structure.
It is also different from density normalized by the maximum number of
edges in a simple DAG, $2|E|/(|V|(|V|-1))$, when $|V|>1$. Both quantities depend on which
vertices and relation types the measurement includes.
\psep
For a target, put $N(v)=|A_G(v)|$. When $N(v)>1$, two complementary
descriptors are
\begin{equation}\label{eq:rho-mu}
 \rho(v)=\frac{S(v)}{N(v)-1},\qquad
 \mu(v)=S(v)-N(v)+1.
\end{equation}
The underlying undirected ancestor graph is connected, because every
ancestor reaches $v$. Hence $S(v)\ge N(v)-1$, and $\mu(v)$ is its undirected
cycle rank. Branches may diverge and reconverge in a DAG, producing an
undirected cycle without a directed logical cycle. For an isolated input,
$\mu=0$ and $\rho$ is left undefined. These descriptors complement depth
and bulk; none determines the entire shape of an argument.

\section{Macro Skeletons and Argument Shape}\label{sec:macro}

\subsection{Projection onto fixed anchors}
Fine-grained measurements depend on decomposition. For a complementary
view of a paper's organization, fix an anchor set $H\subseteq V$ before
comparison. Anchors may include principal results, named constructions,
definitions, and selected external inputs. Exact reprints pointing to one
state have a single representative in $H$. The anchor policy is part of
the measurement contract, rather than a choice made after seeing which
graph looks preferable.

\begin{definition}[Macro projection and skeleton]\label{def:macro}
The macro projection $M_H(G)$ has vertex set $H$. Distinct anchors $a,b$
are joined by $a\to b$ if $G$ has a directed path from $a$ to $b$ whose
internal vertices lie outside $H$. The macro skeleton is
\begin{equation}\label{eq:macro}
 K_H(G)=\TR\bigl(M_H(G)\bigr).
\end{equation}
\end{definition}

\begin{proposition}[Preservation of anchor reachability]\label{prop:macro}
The projection $M_H(G)$ preserves exactly the reachability relation of
$G$ restricted to $H$. Consequently, two finite DAGs with identified
anchor sets and the same reachability relation between those anchors
have the same macro skeleton under that identification.
\end{proposition}
\begin{proof}
A path between anchors in $G$ can be split at every intermediate anchor.
Each resulting segment gives an edge of $M_H(G)$, hence a path in the
projection. Conversely, replacing every projection edge by one of its
supporting paths gives reachability in $G$. The projection is acyclic.
A finite DAG has a unique transitive reduction, namely the cover graph
of the partial order defined by reachability \cite[Theorem 1]{Aho}. The conclusion
follows.
\end{proof}

The result gives a precise condition under which different internal
decompositions yield the same macro skeleton. It does not guarantee that
independent annotators recover the same anchor relations. Branches that
do not contribute to a connection between anchors may disappear from
the projection, and internal differences can be hidden. The full graph
and the supporting paths must therefore remain available.

\subsection{Structural signatures and visual layout}
The shape of an argument is defined by its graph structure, not by screen
coordinates. A structural signature may include height, the distribution
of vertices across depth levels, maximum antichain size, weak component
count, the number of terminal results, branching and merging frequencies,
and overlap between target ancestor subgraphs. Normalized profiles should
be reported together with their original vertex and edge counts.
\psep
Vertices at the same depth form an antichain, but the largest depth level
need not be a maximum antichain. Nor does agreement on a finite collection
of summary statistics imply graph isomorphism. Signatures serve as
descriptors for exploration and statistical comparison; important
differences should be checked against anchor reachability and the source.
\psep
A reading interface may use a compact node-and-edge view with details
revealed on selection. Selecting a target should highlight its ancestors
and the induced dependency edges, while preserving access to source
locations and evidence. Layout layers are presentation choices, not
necessarily the numerical depth. Reproducing a visual arrangement also
requires fixing the layout algorithm and its tie-breaking rules.

\section{From Paper Graphs to a Math Graph}\label{sec:mathgraph}

\subsection{Connecting theorem applications across papers}
A Math Graph connects identified states and inference instances across
manuscripts. A paper-level citation is a starting point for locating a
result, but an accepted application additionally specifies the theorem's
version and statement, its hypotheses, the substitution used, and the
relevant conventions. A theorem with the same name may have a corrected
statement, stronger hypotheses, or a different specialization in another
version.
\psep
Global identifiers should distinguish the manuscript version, the
statement or state, and the proof selection. An importing paper records
an application instance and the evidence that its hypotheses are met.
Two applications may share a source theorem without being the same
computational step. When moving between mathematical settings requires a
translation, that translation and its justification are recorded as part
of the connection. Similarity links can help locate candidates but remain
distinct from accepted uses.
\psep
Circularity in the document citation network does not imply circularity
in a selected theorem proof. Mutual references, revised versions, and
alternative arguments must be separated before combining measurement
graphs. Acyclicity is checked again after integration. An external
dependency that has not been resolved remains at the declared boundary.

\subsection{Axiom-relative and boundary-relative depth}
Expanding external results can, in principle, connect selected arguments
to specified axioms. Axiom-relative depth nevertheless requires a proof
system, inference rules, definition-expansion conventions, library
versions, and proof choices. The name of an axiom system alone does not
determine a unique graph or numerical depth.
\psep
We reserve the description \emph{axiom-relative} for an expansion in which
every source is accounted for by specified axioms, primitive definitions,
or explicit hypotheses, and no external theorem remains an unexpanded
input. Otherwise, the quantity is reported as boundary-relative depth.
For an informal manuscript, correspondence with a formal foundational
representation is itself an additional claim requiring evidence.
A human-oriented substantive step and a kernel-level inference
are different units. Their measurements cannot be compared without an
explicit translation and normalization policy.
\psep
If several proofs of one statement are available, their measurements can
be reported separately. A minimum over this collection is only a minimum
over the available proofs, not over every possible proof of the statement.
Moreover, minimizing the bulk of each target separately need not minimize
the bulk of their union: different proof choices can create or eliminate
shared subarguments.

\section{Applications and Structural Research Questions}\label{sec:applications}

\subsection{Reading support and prerequisite navigation}
The immediate application is a source-linked reading map. A reader
selects a result, inspects its prerequisite subgraph, and marks inputs
already understood. The remaining graph can expose missing definitions,
shared lemmas, and constructions reused by several results. It can also
support alternative reading paths with different intermediate goals.
\psep
This use does not assume that logical order and pedagogical order
coincide. Examples, motivation, and exploratory calculations may be
educationally valuable even when they are not inputs to the final proof.
They should remain accessible as document material with their own
relation types. Learning outcomes must be evaluated through understanding,
retention, and transfer, rather than inferred from the appearance of a
well-organized graph.

\subsection{Single-author and collaborative argument shapes}
An author-associated shape is a hypothesis about recurring organization
in written arguments. It can be investigated by asking whether structural
signatures recur across an author's single-author papers after accounting
for field, problem type, length, period, and the measurement policy.
The object of study is expository and argumentative organization, not
personality or ability.
\psep
Collaborative papers introduce further questions. Their shape may reflect
interaction between different approaches, a division of mathematical
tasks, or a principal organizer of the exposition. It need not be an
average of the collaborators' single-author signatures. Where available
with appropriate consent, documented contribution and writing roles can
be compared with the graph. Author order should not be used as a proxy
for those roles.
\psep
Problem structure is a competing explanation. Different authors working
on closely related questions, different proofs of the same result, and
successive versions of one manuscript provide complementary comparisons.
If authors and topics almost completely coincide in a dataset, their
effects cannot be reliably separated. Repeated collaborative combinations
and samples crossing authors with topics are needed before discussing
stable collaboration patterns. Observational associations alone do not
establish a causal effect of collaboration.

\subsection{Co-use, rare combinations, and cross-field connections}
Co-citation and actual co-use should be recorded separately. Co-citation
in Small's sense concerns documents cited together \cite{Small}. For
our purpose, let $\mathcal U$ be a fixed collection of proof units and
let $\mathcal T$ be a set of identified theorem statements. Define
\begin{equation}\label{eq:co-use}
 X_{ut}=\ind\{t\text{ is used in proof unit }u\},\qquad
 C(t,t')=\sum_{u\in\mathcal U}X_{ut}X_{ut'}\quad(t\ne t').
\end{equation}
The policy must specify whether ``used'' means a direct input, an input
to the selected proof as a whole, or membership in its full ancestor
set. These are different analyses. Theorem identity is retained at the
source level while distinct application instances remain available in
the underlying data.
\psep
The matrix $C$ can be used to investigate groups of theorems repeatedly
used together. Raw counts are strongly affected by theorem popularity
and proof size. A suitable null model should preserve, or otherwise
account for, theorem-use frequencies and the numbers of inputs per
proof unit. Field, period, and available opportunities for joint use are
additional controls. Statistical enrichment is a prompt for mathematical
interpretation, not an explanation of why the theorems belong together.
\psep
Pairs that are rarely used together can also be investigated, but low
co-use does not imply logical incompatibility. A rare theorem, limited
corpus coverage, or a lack of relevant problems may explain the pattern.
Searching many pairs additionally requires attention to multiplicity
and sparse observations.
\psep
To study cross-field connections, first fix an external subject
classification or a reproducible clustering procedure. They are not the
same notion of a field. For a chosen single cluster label $c(v)$, a simple
descriptive quantity for a node with positive outdegree is
\begin{equation}\label{eq:cross-field}
 b_{\mathrm{out}}(v)=
 \frac{|\{w:(v,w)\in E,\ c(w)\ne c(v)\}|}
 {|\{w:(v,w)\in E\}|}.
\end{equation}
It records the fraction of direct uses going outside the node's cluster.
The absolute number and distribution of those uses must also be reported:
one external use and a hundred external uses can have the same fraction.
Multiple subject labels require a separate weighting rule. Paths through
the graph provide additional information about bridging roles, but
centrality can reflect common infrastructure, as emphasized in the
Mathlib analysis \cite[Appendix C.2.5]{Mathlib}. Corpus coverage is an
additional control in comparisons across papers. Centrality is not a
direct measure of mathematical importance.

\subsection{Probabilistic premise suggestions and proof review}
The main corpus-level application proposed here is to use recorded
dependency patterns to guide mathematical scrutiny. Learned premise
selection and proof re-verification already provide precedents
\cite{AlamaPremise}; our question is how reviewed, source-grounded
argument records can support this task for informal manuscripts as well.
The prediction target must be specified before assigning probabilities.
For example, let $J$ denote the set of theorem inputs recorded in a proof
unit sampled under a declared corpus and proof-selection policy, let
$\varphi$ be its target, let $\Gamma$ be its context, and let $K\subseteq J$
be the theorem inputs already observed and accepted as recorded uses.
For a nonempty finite candidate set $T\subseteq\mathcal T\setminus K$,
a future model could estimate
\begin{equation}\label{eq:premise-probability}
 \widehat p_{\mathcal D}(T\mid\varphi,\Gamma,K)
 =\widehat{\mathbb P}_{\mathcal D}
   (T\subseteq J\mid\varphi,\Gamma,K).
\end{equation}
Here $\mathcal D$ denotes the reviewed training corpus together with its
identity, availability, and use-label policies. This expression defines
a possible prediction task, not an estimator implemented or evaluated
in this paper. It predicts recorded joint use under that policy, not
logical necessity, applicability, or the truth of the target statement.
Different valid proof alternatives can have different input sets.
The mechanism by which $K$ is observed or masked is part of the prediction
task and must be fixed when calibration is compared. Merely proposed or
unchecked premises are not assumed to belong to $J$.
\psep
Co-use groups can suggest premise bundles that repeatedly complement
known inputs. Bridge nodes can extend retrieval across subject
communities, while rare configurations can identify cases for closer
inspection. Pairwise counts alone do not determine a joint probability
for a larger bundle, and missing annotations must not automatically be
treated as negative use labels. Low frequency can reflect a legitimate
unusual argument or limited corpus coverage. It is not itself a defect.
\psep
For both human-written and AI-generated proofs, a review interface could
associate each suggested check with supporting proof units and their
source passages. A reviewer then examines the proposed theorem's exact
statement and version, the substitution, all hypotheses, and the local
transition it is meant to justify. This can organize checks for omitted
premises, unproved intermediate claims, and invalid applications. A
candidate repair absent from the original source remains separate from
the reconstructed source graph; it enters only an explicitly identified
revised proof with new evidence. Consequently, suggestions do not alter
the original argument's depth or bulk merely by being recommended.
\psep
The intended benefit is more effective proof review, to be tested by
defect-detection and repair outcomes rather than by familiar-looking
graphs. Corpus agreement, LLM agreement, and DAG validation do not
establish mathematical correctness. Formal checking can supply an
additional validation route when available; correspondence between the
formalized statement and the source argument must still be assessed.
This application extends the research program beyond the discovery,
recomputation, and linking modes specified by PGPP~0.1.

\section{Algorithms and Reproducible Execution}\label{sec:algorithms}

\subsection{Input checks and deterministic calculation}
Before measurement, check unique state identifiers, edge endpoints,
relation types, the proof selection, and the declared boundary. Normalize
only accepted instances of the same calculation. Several evidence records
supporting the same ordered pair produce one measurement edge. A
self-loop after normalization is an error to investigate. The actual
sources of the graph must agree with its declared accepted inputs and
primitives.
\psep
The following procedure operates on an accepted graph. Source
interpretation belongs to the process that constructs its input.

\medskip
\textbf{Algorithm 1. Contract-relative graph measurement.}
\begin{enumerate}
\item Check versions of the manuscript, contract, annotations, and code.
\item Normalize accepted shared calculations and construct the selected
set of dependency edges.
\item Compute a topological order with fixed identifier-based tie-breaking.
Report an unresolved cycle rather than returning DAG measurements.
\item In that order, compute $D(v)$ and $F(v)$ from
\eqref{eq:depth} and \eqref{eq:union}; store $S(v)=|F(v)|$.
\item Propagate declared unresolved-state flags to dependent states.
\item Construct $M_H(G)$ and $K_H(G)$ and check preservation of anchor
reachability.
\item Map measurements back to source occurrences and store input hashes,
output hashes, and the calculation checks.
\end{enumerate}

Unresolved status can occur at internal states, not only at boundary
inputs. If $U$ is the set of states carrying such a flag, define
\begin{equation}\label{eq:incomplete}
 I(v)=\ind\{A_G(v)\cap U\ne\varnothing\}.
\end{equation}
This propagates the effect of declared unresolved material. The value
$I(v)=0$ does not establish that the source is complete or correct; it
only records the absence of currently declared unresolved states in
the ancestor subgraph. Opaque decomposition and an intentionally unexpanded
external boundary should also have separate labels, so that missing
knowledge and the chosen analytical scope are not conflated.

\subsection{Complexity and independent recalculation}
Let $n=|V|$ and $m=|E|$. Given a topological order, depth can be calculated
in $O(n+m)$ time. If hash-set insertion and lookup have expected constant
time, the explicit ancestor-edge unions require expected time
\begin{equation}\label{eq:runtime}
 O\!\left(n+m+\sum_{(u,v)\in E}|F_G(u)|\right).
\end{equation}
This is at most $O(n+m+m^2)$, while storing all such sets requires
$O(n+m+\sum_{v\in V}|F_G(v)|)$ space. Identifier sorting and deterministic
tie-breaking have their own implementation-dependent overhead. Using a
topological order does not make the entire bulk calculation linear.
Bitsets, incremental updates, or persistent shared set representations
are possible alternatives for larger graphs.
\psep
If there are $h$ anchors, searching from each anchor until another anchor
is encountered gives a simple macro projection in worst-case
$O(h(n+m))$ time. Reachability and transitive reduction are then computed
on the projection. Practical reports should specify the algorithms used
and measure memory as well as elapsed computation time. These bounds
concern the graph calculation, not the prior interpretation of a manuscript.
\psep
Independent checks should follow the definitions through a different
calculation route where feasible. Bulk can be checked by first finding
the ancestor vertices and then directly counting induced edges. A
macro skeleton can be checked for reachability preservation and for the
necessity of each retained edge. Such tests establish properties of the
recorded graph; they do not establish that its mathematical interpretation
is faithful to the source.

\subsection{Three levels of reproducibility}
\emph{Computational reproducibility} means obtaining the same measurements
from the same accepted graph. \emph{Annotation reproducibility} concerns
agreement on actual uses when the inventory and decomposition are fixed.
\emph{Discovery reproducibility} concerns agreement when annotators start
from a new manuscript and identify both items and internal steps.
Deterministic calculation addresses the first level. The latter two
require independently annotated reference material and repeated studies.
\psep
Execution records should include the actual tool identifiers, settings,
provided material, source scope, review decisions, and change history.
Unavailable model versions or random seeds must not be invented.
Re-extraction from a manuscript is distinguished from recalculation on
accepted data. A newly discovered dependency creates an evidenced revision,
not an undocumented change to an earlier result.

\section{A Standardized Prompt Protocol}\label{sec:prompt}

\subsection{Purpose and normative status}
To make the framework reusable, we specify the \emph{Paper Graph Prompt
Protocol}, PGPP~0.1. The complete reference prompt appears in
Appendix~\ref{app:prompt}; it can be copied with a source bundle and a run
configuration. It is an operational specification of the paper's method,
not merely a request to draw a dependency graph. Its requirements cover
source identification, decomposition, evidence, review provenance, graph
selection, numerical calculation, and the structure of the returned data.
\psep
PGPP~0.1 is a proposed research protocol, not an externally adopted
standard. Within the protocol, \texttt{MUST} identifies a required condition,
\texttt{SHOULD} allows a documented deviation, and \texttt{MAY} denotes an
optional feature. Conformance concerns adherence to these declared
requirements. It does not establish the truth of an extracted mathematical
claim or guarantee agreement among independent annotators.

\subsection{Three operating modes}
The protocol separates three tasks that should not be conflated in a
reproducibility claim. In \texttt{DISCOVER}, a system constructs proposed
annotations from supplied manuscripts. It resolves its provisional
contract explicitly, and may calculate provisional measurements when
the structural prerequisites are satisfied. New annotations remain
proposed even when the system reports high confidence.
\psep
In \texttt{RECOMPUTE}, the input is an existing frozen contract and graph
snapshot with supplied review decisions. The system recalculates its
quantities without adding states, changing identities, selecting a new
proof, or reinterpreting a source passage. A required semantic repair is
reported as a separate issue. It cannot be hidden inside a supposedly
identical rerun.
\psep
In \texttt{LINK}, the system proposes theorem applications between
supplied Paper Graph snapshots. Each proposal identifies its source
statement, version, substitution, and hypothesis checks. The proposal
does not silently modify an accepted graph. After review, a combined
snapshot may be supplied to \texttt{RECOMPUTE}. This separates the discovery
of a possible cross-paper connection from its use in a reported measurement.

\subsection{Input and output contracts}
The operator supplies the source bundle, scope, boundary policy,
granularity extensions, targets, anchors, and the operating mode. Existing
snapshots and review records are additionally supplied when applicable.
For a first discovery run, the reference profile defaults to the main
mathematical text and all appendices, excluding the Introduction, and
keeps external results opaque. These are editable scope choices, not
universal assumptions about mathematical writing.
\psep
Defaults for targets and anchors must be expanded into explicit ID lists
before calculation. Null means unspecified; an empty list means an
intentional empty selection. Missing anchors prevent a macro comparison,
but do not by themselves prevent depth and bulk calculations under a
complete $(M,B,P,W)$ contract. Missing or unknown premises must not be
silently reclassified as accepted sources of depth zero.
\psep
The response is one JSON object with a fixed record structure. Source
occurrences, contextual states, occurrence/state correspondences, and
inference records remain separate. The selected inference IDs determine
the measurement edges, preventing an independently edited edge list
from disagreeing with the selected joint premises. Other relation types,
sharing decisions, imported reviews, unresolved issues, measurements,
conformance checks, and changes have their own fields.
\psep
The appendix gives the required keys, record fields, and enumerated
statuses. This is a normative field contract, not a claim that a complete
machine-enforced JSON Schema or reference validator has already been
implemented. Structural validation and source interpretation remain
distinct tasks. Identifiers are stable within a versioned snapshot;
fresh semantic extraction is not assumed to recreate the same identifiers
or decomposition automatically.

\subsection{Evidence and calculation discipline}
PGPP~0.1 distinguishes an annotation's evidence status from its review
status. Acceptance is imported from supplied review data with provenance;
the extractor cannot certify its own new annotations as independently
accepted. A graph can therefore have well-defined provisional numerical
properties while its interpretation remains under review. Numerical
consistency does not promote that graph to accepted status.
\psep
The prompt fixes the core operation rules and requires the exact rule
and expression bundle for a domain-specific expansion or normalization.
If the source does not justify a decomposition, the system reports an
opaque block or an unresolved issue. The prompt does not authorize
inventing intermediate arguments merely to eliminate uncertainty.
\psep
Numerical output requires actual graph calculation. Without execution,
the protocol returns unavailable numerical fields and records the checks
as not run. A cycle requires a witness and blocks the affected DAG
measurement. Hashes, tool versions, and execution settings must be
observed values; unavailable metadata are recorded as null with an
explanation. Source coverage, structural checks, numerical checks, and
semantic review are reported separately.

\subsection{Reference conformance cases}
Small abstract graphs make several requirements independently testable.
Table~\ref{tab:prompt-cases} lists reference cases. They define expected
behavior for a future protocol implementation and do not evaluate
extraction from mathematical manuscripts.

\begin{table}[htbp]
\centering\small
\caption{Reference cases for PGPP~0.1. All edges point from input to output.}
\label{tab:prompt-cases}
\renewcommand{\arraystretch}{1.15}
\begin{tabularx}{\textwidth}{@{}p{0.19\textwidth}YY@{}}
\toprule
Case & Specified input & Required behavior\\
\midrule
Shared prefix & $r\to a$, $a\to b$, $a\to c$, $b\to t$, $c\to t$
& $D(t)=3$, $S(t)=5$; the prefix edge is counted once.\\
Alternative proofs & Records $\{a,b\}\Rightarrow t$ and $\{c\}\Rightarrow t$;
select the second, with $c$ a source
& $D(t)=1$, $S(t)=1$; do not add the first record's inputs.\\
Internal uncertainty & $a\to b\to t$ and $U=\{b\}$
& $I(b)=I(t)=1$, although $b$ is not a source.\\
Directed cycle & $a\to b\to a$
& Return the cycle witness and block DAG metrics; do not delete an edge.\\
Anchor refinement & Replace $a\to b$ by $a\to x\to b$; $H=\{a,b\}$
& Both skeletons contain $a\to b$, while fine-grained depth and bulk change.\\
\bottomrule
\end{tabularx}
\end{table}

The shared-prefix case is deliberately stronger than a diamond with no
shared incoming edge: simply adding branch bulks would return six
instead of five. The alternative-proof case checks the distinction
between an AND input set and an OR choice of records. These examples
test the interpretation of the protocol's data contract, not the
correctness of an arbitrary mathematical argument.

\subsection{Supplementary prompt for the interactive website}
\label{sec:viewer-prompt}
For readers' convenience, we provide the following English prompt to
help them explore the framework visually with a few papers of their
choice. After generating source-linked graph records with PGPP~0.1,
readers can supply those records with this prompt to construct an
interactive reading website. The intended use is a quick visual overview
followed by inspection of the prerequisites supporting a selected result.
This supplementary prompt concerns presentation and interaction; it
does not change PGPP~0.1's annotation, proof-selection, or measurement
rules or constitute a separate extraction method.
\psep
For a node $v$, the highlighted nodes and edges
are exactly $A_G(v)$ and $F_G(v)$ from
\eqref{eq:ancestors}--\eqref{eq:ancestor-edges}, with $G$ denoting the
dependency graph supplied to the interface. Thus the highlight covers
the union of all prerequisite paths, including shared branches, rather
than one selected path or only the immediate neighborhood. Equivalence,
similarity, citation-only, and layout relations do not participate in
this traversal.
\psep
The prompt can be supplied with the graph data when constructing the
website. Hover selection is temporary; a clicked selection takes
precedence until explicitly replaced or released.

\noindent\begin{minipage}{\linewidth}
\begin{Verbatim}[fontsize=\footnotesize,baselinestretch=1,frame=none]
INTERACTIVE WEBSITE PROMPT: ANCESTOR HIGHLIGHTING AND CLICK-TO-PIN

Build an interactive dependency-map interface for the supplied graph.
If several Paper Graphs are supplied, allow readers to switch between
them. Use cross-paper dependency links only when explicitly supplied.
Keep supplied review status and unresolved annotations accessible.
An edge A -> B means that A is used in the definition, construction,
or selected justification of B. Use this direction when identifying
prerequisites, regardless of whether arrowheads are drawn.

ALL ANCESTOR PATHS
When the pointer hovers over a node v and no node is pinned, make v,
every direct and indirect prerequisite of v, and all dependency edges
between those nodes glow brightly. Compute the full ancestor closure
by following incoming dependency edges. Highlight the union of all
prerequisite paths, not just immediate neighbors or one chosen path.
Shared prerequisites and shared edges are highlighted only once.
Do not highlight descendants: results subsequently obtained using v
are not part of v's prerequisite highlight. Keep unrelated elements
visually subdued. Ignore auxiliary equivalence, similarity, citation-only,
and layout links when computing the highlight.

PERSISTENT CLICK SELECTION
Clicking a node pins that node and its complete ancestor highlight.
Keep the pinned highlight visible when the pointer leaves the node,
moves over the background, or hovers over another node or an edge.
While a node is pinned, hover events must not replace, extend, or clear
the pinned highlight. Clicking a different node replaces the pinned
target and highlights the new target's complete ancestor subgraph.
Clicking the already pinned node leaves it pinned.

Provide a clearly visible Close or Unpin button while a node is pinned.
Activating that button or pressing Esc clears the pinned target and
its persistent highlight. Resume temporary highlighting on the next
node-hover interaction. When no node is pinned, leaving a hovered node
clears its temporary highlight. Keep hover state and pinned state
separate, with the pinned target taking precedence.

DATA INTEGRITY AND CHECKS
Use the supplied graph; do not infer new dependencies during interaction.
Highlighting and pinning must not modify the graph or its D and S values.
Check a graph with r -> a, a -> b, a -> c, b -> v, c -> v, and v -> w.
Hovering over v must highlight r, a, b, c, v and all five prerequisite
edges, but neither w nor the edge v -> w. After clicking v, hovering
over w or moving off the graph must leave that same highlight pinned.
Clicking b must switch the highlight to r, a, b and their two edges.
Verify that both the Close/Unpin button and Esc release the selection.
Report any interaction checks that have not actually been performed.
\end{Verbatim}
\end{minipage}

These checks specify expected interface behavior. They do not constitute
reported usability results or validation of the graph's mathematical
annotations.

\subsection{Versioning and empirical assessment}
A run records the protocol version, resolved contract, exact supplied
material, and any extension or deviation. Changing the meaning of a
step, an edge, a field, or a default requires an explicit protocol or
profile revision. Cosmetic prompt changes should also be retained in
the run record, so that ``the same prompt'' has a concrete meaning.
\psep
The effect of standardization is an empirical question. A study can
compare an unrestricted graph request, the complete reference prompt,
and controlled variants omitting selected requirements. It should
evaluate valid-output rate, structural conformance, source-grounded
annotation agreement, metric error, and unresolved-case reporting
separately. Identical wording alone does not guarantee reproducible
semantic extraction. The protocol instead makes the decisions and
failure conditions visible enough to compare such runs.

\section{Evaluation Design}\label{sec:evaluation}

\subsection{Reference data and extraction quality}
A useful benchmark should include several mathematical fields and
argument types: calculation-heavy proofs, conceptual arguments, texts
with substantial omissions, and papers depending on many external
results. Material used to develop the registry should be separated from
the final evaluation. Closely related manuscript versions and proof
families should remain in the same partition to avoid misleadingly easy
generalization tests.
\psep
At least two suitably qualified annotators should work independently on
the reference material. Disagreements are adjudicated with source evidence,
and genuinely unresolved relations remain marked as such. Sample sizes
and target precision should be specified after a pilot establishes the
variation in annotation and structural measurements. A reference graph
is a documented expert judgment, not an assumption that informal
mathematics has an unambiguous canonical decomposition.
\psep
Comparison conditions should separate extraction from raw text, annotation
with a fixed item inventory, and annotation with both an inventory and
a fixed decomposition. Citation- and cross-reference-based graphs provide
baselines. Where available, author annotations and the outputs of existing
tools can be compared on the same material and relation definitions.
Repeated runs under fixed settings and runs after tool changes should
be reported separately. Table~\ref{tab:evaluation} summarizes the proposed
evaluation dimensions.

\begin{table}[htbp]
\centering\small
\caption{Proposed evaluation dimensions. No empirical values are asserted
in this paper.}\label{tab:evaluation}
\renewcommand{\arraystretch}{1.15}
\begin{tabularx}{\textwidth}{@{}p{0.19\textwidth}YY@{}}
\toprule
Object & Main quantities & Required accompanying checks\\
\midrule
Source inventory & Precision and recall by item type & Numbered and unnumbered items assessed separately\\
Actual-use edges & Precision, recall, and F1 & Supporting passages, contexts, and application hypotheses\\
Decomposition & Step alignment and boundary agreement & Shared calculations and excessive subdivision\\
Macro skeleton & Anchor ancestor-set agreement & Height, width, branching, and terminal structure\\
Depth and bulk & Absolute and relative error distributions & Boundary changes and unresolved components\\
Reproducibility & Structural and numerical agreement & Recalculation distinguished from new extraction\\
Learning support & Understanding, retention, and transfer & Prior knowledge, task allocation, and source navigation\\
Proof review & Premise retrieval, calibration, and defect detection & Held-out proof families, checked applicability, and false alarms\\
\bottomrule
\end{tabularx}
\end{table}

If state correspondence is uncertain, alignment error must be evaluated
before edge agreement. Arbitrary merging can artificially increase the
reported agreement. Besides averages, report paper-level distributions,
failure cases, and sensitivity to external boundaries and opaque blocks.
Stability across tools and agreement with the expert reference are
different properties.
\psep
Premise suggestions require a separate evaluation. Held-out splits should
respect manuscript versions, near-duplicate statements, and proof families,
and retrieval should enforce the declared availability boundary. Compare
co-use and graph-based suggestions against theorem-frequency and semantic
retrieval baselines on the same candidate pool. Report premise precision,
recall, ranking quality, and calibration for the specified use outcome;
these do not establish improved proof checking. A review-assistance study
should separately measure confirmed omissions and invalid applications
found, successful repairs, false alarms, and reviewer effort. It should
include human-written and AI-generated proofs, with genuine unusual
proofs as well as independently adjudicated defects. This tests whether
statistical suggestions help mathematical review without penalizing
valid departures from common patterns.

\subsection{Learning support}
Learning studies can compare ordinary document reading with graph-assisted
reading while accounting for prior knowledge and task difficulty. Outcomes
may include explaining a target result, identifying its assumptions,
applying an idea in a new setting, and delayed retention. Task allocation
should control the carryover caused by reading the same argument twice.
\psep
The graph should also be evaluated as a navigational representation.
Readers may be asked to locate the source of a prerequisite, distinguish
an accepted input from an unresolved dependency, or identify a shared
construction. Successful navigation is useful evidence about the
interface, but does not by itself show improved mathematical understanding.
The study should record how graphs were prepared and whether the version
presented to the reader contained unresolved annotations.

\subsection{Author, collaboration, and reuse hypotheses}
Studies of author-associated signatures need matched comparisons across
topics and lengths, repeated papers by the same author, and related
problems addressed by different authors. Graph annotation can be blinded
to author identity. Papers, rather than their many nodes, are the primary
units for claims about paper-level organization; dependence within
authors, collaborations, and proof families must be modeled or respected
in resampling and data splits.
\psep
Collaboration analyses should ask whether an observed signature remains
stable after accounting for the mathematical problem and, where known,
writing roles. Sufficient replication is needed to distinguish a recurring
collaborative pattern from one unusual paper. The goal is to test whether
an organizing tendency is identifiable, not to assign a permanent graph
signature to a person.
\psep
Co-use and cross-field analyses should state the proof-unit boundary,
theorem-identity policy, clustering resolution, and time window. Null
models, multiplicity adjustments, and sensitivity to corpus coverage are
needed before interpreting unusual combinations. Candidate connections
should then be examined mathematically. A statistical pattern can suggest
a question without implying that a new theorem or proof exists.

\section{Limitations and Scope}\label{sec:limits}
The central limitation is the interpretation of informal mathematics.
Detailed source evidence does not recover every omitted argument or
settle every disagreement about a step boundary. A more specialized
registry may improve consistency while reducing transfer across domains.
Its suitability is therefore an empirical question rather than a
consequence of deterministic graph algorithms.
\psep
The stability of a macro skeleton is conditional on fixed anchors and
their reachability. Different anchor choices can yield different
skeletons, and identical skeletons can conceal different internal
arguments. Even exact agreement of the recorded graphs does not establish
agreement of all the mathematics that the manuscript leaves implicit.
\psep
Depth and bulk do not measure difficulty, originality, truth, or
importance. A short step may invoke a profound idea, while a large
subgraph may consist of routine calculations. A different proof may
change the graph substantially. These observations are reasons to keep
representation choices explicit and to study structural quantities
jointly with mathematical context.
\psep
Author and collaboration studies are additionally limited by confounding,
corpus selection, and the availability of reliable role information.
Observed writing patterns should not be interpreted as personality traits
or individual merit. Co-use and bridge statistics are likewise properties
of a specified corpus and relation policy, not of all mathematics.
\psep
Source access, the distribution of derived annotations, and the handling
of unpublished manuscripts require appropriate permissions. Version
changes and corrections must remain traceable. In particular, an accepted
application of a theorem is tied to its statement, assumptions, and
context; it is not an unrestricted approval of every later use.
\psep
This paper establishes definitions, conditional graph properties, and
an evaluation design. It does not report measured extraction performance,
stable author signatures, demonstrated learning benefits, calibrated
premise probabilities, or improved detection of proof errors. Each of
those claims requires the independent empirical work described in
Section~\ref{sec:evaluation}.

\section{Conclusion}\label{sec:conclusion}
Paper Graph and Math Graph provide a proposed framework for connecting
source occurrences, substantive argument steps, and theorem applications
across mathematical papers. Their measurements are explicitly relative
to a manuscript, boundary, granularity policy, and proof selection.
Set-based bulk preserves shared subarguments, while depth records a
longest chain of dependence. A fixed-anchor macro skeleton supports a
coarser comparison whenever anchor reachability is preserved.
\psep
The PGPP~0.1 reference prompt provides an operational entry point, with
explicit inputs, modes, output records, and conformance criteria. Its
benefit for annotation agreement remains to be established empirically.
\psep
The framework turns a reading map into a possible research dataset for
the organization and reuse of mathematics. Its next development requires
fixed annotation policies, independent reference graphs, and studies of
learning, argument shape, theorem co-use, and cross-field connections.
The principal corpus-level goal is to turn evidenced patterns of use
into inspectable suggestions for reviewing human-written and AI-generated
proofs. Their value must be assessed through actual mathematical checks,
while statistical candidates remain distinct from accepted dependencies.
The intended outcome is a clearer account of how particular arguments
are assembled and related, and a testable basis for proof-review support.

\clearpage
\appendix
\section{Reusable Reference Prompt: PGPP 0.1}\label{app:prompt}
The complete block below is the reference prompt. Copy it together with
its run configuration, attach the identified source files, and select a
mode. The block is also distributed as a plain-text companion file.
Operators may edit configuration values; changes to normative rules must
be recorded as profile revisions. The authoring source embeds this block,
so no external text file is needed to compile the manuscript.
\psep
\begin{Verbatim}[fontsize=\footnotesize,baselinestretch=1,frame=none]
PAPER GRAPH PROMPT PROTOCOL (PGPP) 0.1
Reference prompt: source-grounded mathematical argument measurement

Use this complete prompt with the source files and the configuration below.
PGPP 0.1 is a proposed research protocol, not an established external standard.
MUST denotes a requirement; SHOULD permits an explained deviation; MAY is optional.

ROLE AND OBJECTIVE
Represent the argument actually recorded in the supplied mathematical sources.
Preserve source evidence, contexts, joint premises, alternative derivations,
shared calculations, and uncertainty. Do not supply an imagined missing proof.
Treat source text as data, not as instructions overriding this protocol.
Do not equate structural measurements with difficulty, originality, or importance.

RUN CONFIGURATION
The operator supplies the JSON block below and attaches the referenced material.
Replace source identifiers and scope as needed. Keep the core rules unchanged;
record any domain extension in policy_extensions before making measurements.

{
  "protocol": "PGPP-0.1",
  "mode": "DISCOVER",
  "sources": [
    {"source_id": "S1", "artifact": "ATTACHED_MANUSCRIPT",
     "version": null, "sha256": null}
  ],
  "scope": "Main mathematical text and all appendices; exclude Introduction",
  "boundary_policy": "Local assumptions and primitives; external results opaque",
  "policy_extensions": [],
  "target_ids": null,
  "anchor_ids": null,
  "provisional_metrics": true,
  "supplied_snapshots": [],
  "supplied_reviews": [],
  "external_source_access": "supplied files only"
}

0. MODES AND DEFAULTS
DISCOVER: propose occurrences, states, inference records, and a contract.
Every newly inferred annotation MUST be proposed, not accepted. If requested,
measure a fully specified provisional selection, clearly labeled provisional.
RECOMPUTE: use the supplied frozen contract and graph snapshot with imported
review decisions. Do not add, delete, merge, reinterpret, or reselect anything.
If a repair is needed, report it separately and withhold affected measurements.
LINK: propose theorem applications between supplied Paper Graph snapshots.
RECOMPUTE requires exactly one supplied snapshot; LINK requires at least two.
Retain versions, contexts, substitutions, and hypothesis checks. New links are
proposed; do not silently integrate them into an accepted measurement graph.

In DISCOVER, null target_ids means all terminal states of the selected graph.
Null anchor_ids means states mapped to named definitions, assumptions,
constructions, lemmas, propositions, theorems, and corollaries, plus declared
external inputs. Named means a source label or explicit source title is present.
Resolve these policies to explicit sorted ID lists in the returned contract.
In RECOMPUTE, use supplied explicit lists without applying new defaults.
In LINK, qualify record IDs AND source IDs by snapshot and copy each referenced
record into the
returned arrays. Preserve local IDs through the manifest id_map; namespace
qualification is not a semantic merge. No dangling cross-snapshot references.
An explicit empty anchor list requests no macro analysis. Missing anchors in a
recompute snapshot block macro analysis, not otherwise valid depth/bulk analysis.
Use the supplied artifact as the source version; do not guess a publication date.
If more than one candidate version is present, report the ambiguity before
mixing records. Never claim a digest was calculated unless it actually was.

1. FREEZE THE ANALYTICAL CONTRACT
Record C = (M, B, P, W, H): manuscript/version/scope M; boundary state IDs B;
operation registry and decomposition policy P; selected inference IDs W;
anchor state IDs H. Also record target IDs and all changes from supplied data.
Before numerical calculation, resolve B, W, and the measured state set explicitly.
In a discovery run this resolution is provisional, not external acceptance.
Record the exact source ranges inspected and ranges missing or unreadable.
A complete inventory of numbered items does not imply complete argument coverage.

2. INVENTORY AND IDENTITIES
Inventory all in-scope labeled definitions, assumptions, results, and numbered
equations. Include unnumbered claims, constructions, and expressions used in
the argument. Preserve labels and exact source locations, including appendices.
Keep document occurrences distinct even when they express equivalent statements.
Represent computational states separately, recording context, claim/expression,
objects, assumptions, substitution, and conventions. Preserve a many-to-many
occurrence/state map. Do not identify one component with an entire formula bundle.
Use IDs based on source/version and a stable local identifier, not formula text
alone. Copy supplied IDs unchanged in RECOMPUTE. Record any renaming explicitly.

3. INFERENCE AND DECOMPOSITION RULES
An input-to-output edge means actual use in the chosen argument. Earlier position,
nearby numbering, citation, and semantic similarity alone do not establish use.
Record a joint input set in ONE inference record (AND). Keep alternative records
for the same output separate (OR). Select exactly one generating record for each
derived state in the measured graph; boundary states have none.
Never combine the premises of alternative proofs merely to obtain one graph.
If the source presents several alternatives without selecting one, record them
and leave the affected selection unresolved unless the operator supplies a choice.
Shared states must have compatible contexts and selected generating records.

Freeze the following operation rules as the core registry:
SUB: one specified simultaneous substitution map on a specified expression bundle.
APPLY: one identified theorem/identity instance with its substitution and premises.
EXPAND: one specified expansion rule on one declared expression bundle.
COLLECT: one normalization pass with a declared normal form and expression bundle.
COMPONENT: extraction of one used component, or a jointly used declared component set.
CONSTRUCT: one defined object/map with its required construction conditions.
COMPAT: one independently justified compatibility or applicability condition.
COMBINE: one specified inference joining previously obtained outputs.
OPAQUE: an undecomposed block; record its known inputs, output, and unresolved parts.
Independent outputs requiring separate justification MUST be separated.
Do not split by printed lines, equation length, or number of replaced symbols.
Keep independent calculations parallel; add sequence only for actual output use.
For EXPAND/COLLECT and other domain-specific operations, name the exact rule and
bundle. If no applicable rule can be specified, retain OPAQUE and report the gap.
Every step MUST record its inputs, output, operation, source location, and a
brief externally checkable reason for its step boundary. Do not generate a
hidden derivation and present it as a calculation appearing in the source.

4. EVIDENCE, REVIEW, SHARING, AND UNCERTAINTY
Distinguish explicit source evidence, locally justified interpretation, and
unresolved evidence. Store a short source-grounded rationale for each inference.
Review status is a separate field: proposed, accepted, rejected, or unresolved.
Only supplied review decisions may confer accepted status. Retain reviewer and
review-record provenance. Never invent a reviewer or mark your own new output
as independently accepted. Conflicting reviews remain unresolved unless supplied adjudication
explicitly supersedes them. Confidence scores cannot substitute for evidence.
Repeated source displays may share a state only when context, objects,
substitutions, inputs, and the calculation instance agree. Record the evidence
for sharing; equivalence or formula similarity is insufficient.
Missing premises MUST NOT silently become accepted depth-zero inputs.
An operator-declared external result may remain a boundary input. A newly
proposed boundary is provisional and MUST be identified as such.
Unknown inputs, unavailable passages, ambiguous identities, and unsupported
applications produce issues and unresolved flags on affected states.
OPAQUE is not a claim that the block's interior was recovered. Known-input
opacity, deliberate external scope, and unknown premises are separate flags.
Propagate declared unresolved status from internal states as well as sources.
Known-input opaque blocks can be measured at their declared granularity. Unknown
premises block the affected metric scope unless the operator explicitly declares
them unresolved boundary inputs; their unresolved status must still propagate.

5. BUILD AND MEASURE THE SELECTED STATE GRAPH
Project each selected inference r into edges (u, output(r)) for u in inputs(r).
Retain inference records. The measured vertices are the union of B, target IDs,
and inputs/outputs of selected records. Do not include unselected alternatives
merely because they occur in the inventory. Count each ordered pair once.
Do not add operation vertices to the depth graph. Do not discard directly used
transitive edges from the measurement graph. Keep other relation types separate.
Check ID uniqueness, referential integrity, one selected generator per derived
state, source/boundary equality, and acyclicity before computing metrics.
Record a directed cycle witness when one exists. Do not delete an edge to hide it.

For each vertex v, let Pred(v) be its direct input states. Define:
A(v) = {v} union all vertices reaching v.
F(v) = the edges induced by A(v).
D(v) = 0 for a source; otherwise 1 + max D(u) over u in Pred(v).
F(v) = union over u in Pred(v) of [F(u) union {(u,v)}].
S(v) = |F(v)|; N(v) = |A(v)|.
q(G) = |E|/|V| for nonempty V, otherwise null.
rho(v) = S(v)/(N(v)-1) if N(v)>1, otherwise null.
mu(v) = S(v)-N(v)+1; thus mu(v)=0 for an isolated input.
I(v) = 1 iff A(v) intersects the declared unresolved-state set U.
For targets T, report the union of F(t), not the sum of target bulks.
Depth is an edge count. Bulk is a distinct-edge count, not an operation count.

For macro analysis, fix H. Connect a to b when a path has no internal H vertex.
Take the transitive reduction of this projection to obtain K_H(G).
Retain supporting paths. Check anchor reachability preservation and the necessity
of each retained skeleton edge. Do not use macro edge counts as basic bulk.
Label all measures boundary-relative unless external justifications have been
expanded to specified foundational inputs under a declared proof-system policy.

Use executable graph calculations when available. Record the actual computation
method, code/version, and checks. Without execution, return numeric fields as
null and mark them not_run; do not guess values or claim code was executed.
Blocked components MUST NOT contribute fabricated values to whole-graph totals.
If only a restricted subgraph is measurable, identify it as a new explicit scope;
do not label its values as the measurements of the original whole graph.

6. REQUIRED OUTPUT CONTRACT
Return one UTF-8 JSON object with the following top-level keys, in this order:
protocol, mode, run_status, manifest, contract, occurrences, states,
occurrence_state_links, inferences, other_relations, sharing_decisions,
reviews, issues, measurement, checks, change_log, summary.
Use null for unknown or inapplicable scalar values and explain missing required
data in issues. Use [] for empty collections. Do not use NaN, Infinity, or ellipses.
All IDs are strings; lists of IDs reference records in this object or explicit
source/snapshot IDs in the manifest. JSON field names below are normative.
Sort records with id by id and sources by source_id. Sort set-valued ID and
edge lists lexicographically and deduplicate them. Preserve path order, source
range order, hypothesis order, and supplied review chronology. Other arrays
retain source order; document any tie-breaking rule.
Do not invent a universal canonical JSON hash. Hash original files only when read;
hashes of serialized graph files must identify the actual bytes/format used.

protocol: "PGPP-0.1".
mode: DISCOVER | RECOMPUTE | LINK.
run_status: complete | partial | blocked (execution/coverage status, not truth).
manifest: {sources, inspected_ranges, unavailable_ranges, tool_id, tool_version,
           execution_settings, computation_method, code_id, input_hashes,
           output_hashes, parent_snapshot_ids, id_map}.
  id_map entries: {snapshot_id, original_id, returned_id}; empty for unchanged IDs.
  sources entries: {source_id, artifact, version, sha256}.
contract: {contract_id, status, M, B, P, W, H, target_ids}.
  status: provisional | frozen. M identifies sources and scope; B is state IDs;
  P records core version, exact domain rules, and extensions; W is inference IDs;
  H and target_ids are state ID lists. An unresolved list is null, not [].
occurrences entries: {id, source_id, kind, label, locator, text}.
  locator: {page, section, item, span}; absent coordinates may be null, but at
  least one precise supplied-source locator is required for an extracted item.
states entries: {id, context, claim, objects, assumptions, substitution,
                 conventions, boundary_kind, unresolved, opaque}.
  boundary_kind: null | assumption | primitive | external | unresolved_input.
  unresolved and opaque are booleans. Do not use opaque as a synonym for false.
occurrence_state_links entries: {id, occurrence_id, state_id, role}.
  role: whole | component | internal_step | exact_reprint.
inferences entries: {id, inputs, output, operation, rule, expression_bundle,
                     proof_instance_id, application, evidence, evidence_status,
                     review_status}.
  operation is a core operation code; inputs are joint state IDs.
  proof_instance_id records the source proof alternative; separate local proofs
  can have different IDs. Compatibility requires a consistent choice at shared
  states and any source-stated constraints between alternatives.
  application is null or {source_id, theorem_locator, substitution,
                          hypothesis_checks, source_state_id}.
  hypothesis_checks entries: {hypothesis, supporting_state_ids, status}.
  status: established | unresolved; established requires supplied review evidence
  in an accepted graph and is only a proposal in a provisional graph.
  evidence entries: {source_id, locator, rationale}.
  evidence_status: explicit | interpreted | unresolved.
  review_status: proposed | accepted | rejected | unresolved.
other_relations entries: {id, from_id, to_id, type, evidence}.
  type: citation | similarity | equivalence | exposition.
sharing_decisions entries: {id, occurrence_ids, state_id, rationale,
                           review_status, review_ids}.
reviews entries: {id, record_id, reviewer, decision, provenance}.
  record_id may be a state, occurrence_state_link, inference, sharing, or contract ID.
  Superseding review provenance MUST identify the review IDs it replaces.
issues entries: {id, code, affected_ids, source_locator, description,
                 blocks, resolution_needed}.
  blocks is a list drawn from inventory, selection, metrics, macro, coverage.
measurement: {graph_basis, status, vertex_ids, edges, node_metrics,
              graph_q, target_status, target_union_edges, target_union_bulk, macro}.
  target_status: computed | blocked | not_run | not_requested; use not_requested
  for an explicit empty target list. Target failure does not erase valid node
  metrics, but requested whole-target totals remain null when unavailable.
  graph_basis: provisional | accepted | none.
  status: computed | blocked | not_run. Edges are [input_id, output_id] pairs.
  node_metrics entries: {id, D, S, N, rho, mu, I} using the state ID.
  macro: {status, anchor_ids, projection_edges, skeleton_edges, witnesses}.
  macro.status: computed | blocked | not_run | not_requested.
  Use not_requested for H=[]; missing H blocks requested macro analysis.
  Unavailable graph/edge lists are null, not the empty graph.
  witnesses entries: {id, edge, path}; path is a list of original state IDs.
  Supply one path per projection edge; it has the stated endpoints, uses actual
  graph edges, and has no internal anchor. Skeleton edges are a subset of
  projection edges and inherit those witnesses.
checks entries: {id, category, status, affected_ids, method, evidence}.
  category: structural | numerical | semantic | coverage.
  For a failed acyclicity check, evidence MUST contain {"cycle": [state IDs]},
  with the first ID repeated at the end and every successive edge present.
  status: pass | fail | not_run | not_applicable | unresolved.
change_log entries: {id, affected_ids, previous_value, new_value, reason}.
summary: a brief account of scope, reading targets, unresolved items, and limits.

Strings hold mathematical expressions in source notation; escape them for JSON.
Context, objects, assumptions, substitution, conventions, rule, and
expression_bundle MUST be explicit strings or structured JSON, never silent
assumptions. Source and inspection ranges MUST use the same locator convention.
Unknown review provenance means no accepted status. Require imported acceptance
for selected inferences, state identities, occurrence/state mappings, sharing,
and the frozen contract before
using graph_basis=accepted. All other measured graphs are provisional.

7. CONFORMANCE REPORT AND TERMINATION
Return one check for each: source/version identification; scope coverage;
ID/reference validity; context and occurrence mapping; operation-rule coverage;
proof-selection compatibility; boundary/source equality; derived-edge equality;
acyclicity; shared-edge union; depth recurrence; unresolved-status propagation;
macro reachability; macro minimality; imported review provenance.
Report unexecuted checks as not_run. A structural pass is not a semantic pass.
In LINK, give proposed application records and mark unrequested metrics
not_run. In RECOMPUTE, unsupported new semantic choices are issues, not repairs.
If output limits prevent completion, return a valid partial JSON object with
the completed source ranges and explicit remaining work. Do not return dangling
references or silently omit records; withhold affected calculations.
Do not claim identical prompt text guarantees identical semantic extraction.
Conclude inside summary with the graph basis, coverage status, and outstanding
issues. Do not claim protocol conformance establishes mathematical correctness.
\end{Verbatim}

\end{document}